\documentclass[10pt,twocolumn,letterpaper]{article}

\usepackage[pagenumbers]{cvpr} 

\definecolor{cvprblue}{rgb}{0.21,0.49,0.74}
\usepackage[pagebackref,breaklinks,colorlinks,allcolors=cvprblue]{hyperref}

\usepackage{rotating}
\usepackage{microtype}
\usepackage{booktabs}    
\usepackage{tikz}
\usepackage{caption}     
\usepackage{multirow}    
\usepackage{longtable}   
\usepackage{siunitx}
\usetikzlibrary{positioning, arrows.meta, decorations.pathmorphing}
\usepackage{algorithm}
\usepackage{algpseudocode}
\usepackage{setspace}
\usepackage{amsthm}
\theoremstyle{plain}
\newtheorem{theorem}{Theorem}
\newtheorem{proposition}[theorem]{Proposition}

\def\paperID{*****} 
\def\confName{Arxiv}
\def\confYear{2026}

\title{Likelihood-Separable Diffusion Inference for Multi-Image MRI Super-Resolution}

\author{
    Samuel~W.~Remedios\\
    Johns Hopkins University\\
    Baltimore, MD, USA\\
    {\tt\small samuel.remedios@jhu.edu}
    \and
    Zhangxing~Bian\\
    Johns Hopkins University\\
    Baltimore, MD, USA\\
    \and
    Shuwen~Wei\\
    Johns Hopkins University\\
    Baltimore, MD, USA\\
    \and
    Aaron~Carass\\
    Johns Hopkins University\\
    Baltimore, MD, USA\\
    \and
    Jerry~L.~Prince\\
    Johns Hopkins University\\
    Baltimore, MD, USA\\
    \and
    Blake~E.~Dewey\\
    Johns Hopkins University\\
    Baltimore, MD, USA\\
}

\begin{document}

\twocolumn[{
\renewcommand\twocolumn[1][]{#1}  
\maketitle
\begin{center}
    \includegraphics[width=\textwidth]{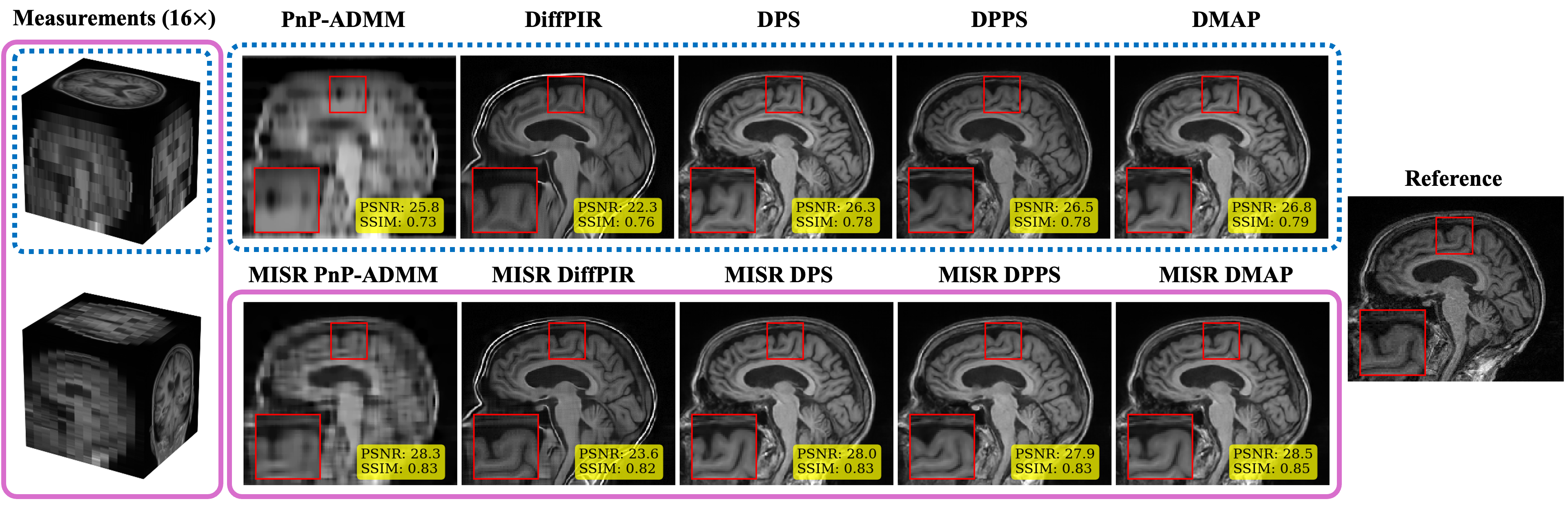}
    \captionof{figure}{
        \textbf{Overview of our proposed MISR generalization to diffusion super-resolution.} Our method generalizes existing methods for multi-image super-resolution~(MISR). The first column shows low-resolution~(LR) anisotropic MRI measurements with 16$\times$ downsampling along orthogonal planes. For super-resolution results, the first row shows the existing methods for single-image super-resolution~(SISR), which only use the LR measurement enclosed in the blue dotted box. The second row shows results for our proposed generalization of each method for MISR, which use both LR measurements enclosed in the pink solid box. All methods used the same 3D volumetric diffusion model, and all results are 3D volumetric images. A single sagittal slice is shown for viewing clarity. Volumetric PSNR and SSIM are overlaid in the yellow box, and zoomed inset regions are shown in red. MISR improves performance for all methods and better estimates the anatomy of the reference image, shown on the far right.
    }
    \label{fig:teaser}
\end{center}
\vspace{1em}
}]

\begin{abstract}
Diffusion models are the current state-of-the-art for solving inverse problems in imaging.
Their impressive generative capability allows them to approximate sampling from a prior distribution, which alongside a known likelihood function permits posterior sampling without retraining the model.
While recent methods have made strides in advancing the accuracy of posterior sampling, the majority focuses on single-image inverse problems.
However, for modalities such as magnetic resonance imaging~(MRI), it is common to acquire multiple complementary measurements, each low-resolution along a different axis.
In this work, we generalize common diffusion-based inverse single-image problem solvers for multi-image super-resolution~(MISR) MRI.
We show that the DPS likelihood correction allows an exactly-separable gradient decomposition across independently acquired measurements, enabling MISR without constructing a joint operator, modifying the diffusion model, or increasing network function evaluations.
We derive MISR versions of DPS, DMAP, DPPS, and diffusion-based PnP/ADMM, and demonstrate substantial gains over SISR across $4\times/8\times/16\times$ anisotropic degradations.
Our results achieve state-of-the-art super-resolution of anisotropic MRI volumes and, critically, enable reconstruction of near-isotropic anatomy from routine 2D multi-slice acquisitions, which are otherwise highly degraded in orthogonal views.
\end{abstract}

\section{Introduction}
\label{sec:intro}
Super-resolution~(SR) is a linear inverse problem that estimates a candidate high-resolution~(HR) image given an observed low-resolution~(LR) image~\cite{Buades-2005-NLM,rousseau2008brain,Rousseau-2010-NLMIntermodality,Jafari-Khouzani-2014-NLM,Manjon-2010-SRSelfSimilarity,Zhang-2015-Sparse,willoughby1979solutions,Xin-2008-Tikhonov,lepcha2023image,Woo-2012-TongueSR,Dong-2016-SRCNN,lu2022transformer,yang2020learning,yan2021smir,li2022transformer,feng2021task,zhou2023hybrid,forigua2022superformer,cheng2023hybrid,fang2022hybrid,wu2023super,wang2023inversesr,wu2024anires2d,chang2024high,han2024arbitrary,wu20243d,bora2017compressed,menon2020pulse,daras21_ilo,daras22-score-ilo}.
Mathematically, SR uses the observed LR image $y\in\mathbb{R}^m$~(the \textit{measurement}) to estimate the HR image $x\in\mathbb{R}^n$ given the forward model 
\begin{equation}
    y = Ax+\epsilon,
\label{eq:forward}
\end{equation} where $A\in\mathbb{R}^{m\times n}$ is the linear forward operator and $\epsilon \sim \mathcal{N}(0, \sigma\mathbf{I}_m)$.
Since $m < n$, Eq.~\ref{eq:forward} is underdetermined and noisy, multiple solutions exist; the problem is \textit{ill-posed}.
Some form of regularization is required to select a solution.

While classical approaches use hand-crafted regularization priors such as total variation~\cite{willoughby1979solutions,Xin-2008-Tikhonov} or sparsity under some transform~
\cite{donoho2006compressed,candes2006robust,sen2009compressive,saafin2015image}, contemporary data-driven approaches use denoisers as a plug-and-play prior~\cite{wang2023ddnm,song2023pseudoinverse,kawar2022denoising,kawar2021snips,choi_ILVR}.
This practice began with generic denoising algorithms~\cite{venkatakrishnan2013plug,chan2016plug,romano2017little} but has recently found deep theoretic connections with the family of diffusion models, including denoising diffusion probabilistic models~(DDPMs)~\cite{ho2020denoising}, score-based models~\cite{songscore}, flow models~\cite{liuflow,gao2025diffusionmeetsflow}, and Schr{\"o}dinger bridges~\cite{de2021diffusion,shi2023diffusion}.
These approaches learn a generative model that transforms samples from one distribution~(often a high-dimensional standard Gaussian) to the data distribution.
In this way, diffusion models act as a Bayesian prior on the desired data distribution $p(x)$.
Accordingly, inverse problem solving algorithms leverage Bayes' rule to sample from the posterior:
\begin{equation}
    p(x|y) \propto p(y|x)p(x),
\label{eq:bayes}
\end{equation}
where the likelihood term $p(y | x)$ is approximated by the known measurement function Eq.~\ref{eq:forward}.


This insight has spurred great developments in the use of diffusion models for inverse problem solving.
The recipe for solving any inverse problem with Eq.~\ref{eq:bayes} requires only three ingredients: a known forward operator $A$, a pretrained diffusion model to approximate $p(x)$, and a technique to sample from the posterior $p(y|x)p(x)$.
Most works assume that $A$ is known, assume the distribution of $\epsilon$ is known, and use an existing well-trained diffusion model.
Thus, algorithm development focuses on the sampling procedure.
Diffusion Posterior Sampling~(DPS)~\cite{dps}, which corrects samples on the diffusion trajectory towards the data-consistent manifold via a gradient update, is possibly the most well-known.

Most work on SR has focused on single-measurement inverse problem solving.
However, in problem settings such as magnetic resonance imaging~(MRI), multiple measurements are common.
For example, an MRI session may include multiple scans along orthogonal axes.
These occur for a multitude of reasons, such as improving signal-to-noise ratio, mitigating motion artifacts, or due to special MRI pulse sequences designed for slice-wise acquisition.
Such sessions often result in multiple image volumes that are HR along two axes~(called \textit{in-plane}) but LR along the third axis~(the \textit{through-plane}).
These \textit{anisotropic} image volumes are faster to acquire, but by trading signal-to-noise~(SNR) and in-plane resolution for a degraded through-plane resolution, the image volumes are not clinically useful when viewed through-plane.
Radiologists therefore do not read anisotropic volumes in the through-planes, as important anatomical structures are blurred and/or decimated.

To leverage the complementary information available in this real-world setting, we propose to generalize diffusion inverse problem solving approaches for the multi-image super-resolution~(MISR) problem.
In this paper, we show that diffusion models can accurately estimate the underlying isotropic anatomy without retraining for specific resolutions.
To our knowledge, no prior work in MRI or diffusion inverse problems has shown that multiple anisotropic volumes can be restored to isotropy.
This positions MISR diffusion as a practical, high-impact alternative to long 3D MRI scans.

Our main contributions are as follows:
\begin{itemize}
    \item \textbf{Likelihood separability for diffusion posterior sampling.}
    We show that the likelihood gradient decomposes exactly across independent measurements, permitting MISR correction using only per-measurement operators.
    This enables multi-image posterior correction without a joint operator or modification of the diffusion model.
    \item \textbf{A unified MISR generalization of diffusion inverse solvers.}
    We derive MISR versions of DPS, DMAP, DPPS, and diffusion-based PnP/ADMM.
    \item \textbf{Noise weighting for heterogeneous MRI acquisitions.}
    We introduce inverse-variance weighting of per-measurement gradients, enabling principled fusion of measurements with different slice thicknesses, SNRs, and spatial resolutions.
    \item \textbf{State-of-the-art MISR MRI performance.}
    Across $4\times/8\times/16\times$ through-plane degradations, our MISR extensions substantially outperform their single-image counterparts, improving PSNR by 1 to 3 dB while producing anatomically faithful reconstructions.
    Evaluations were conducted with $64$ network function evaluations~(NFEs) for the reverse diffusion process, yielding a total super-resolution time per volume of less than 60 seconds on an NVIDIA RTX 6000 ADA.
\end{itemize}
When taken together, our contributions establish diffusion models as a practical, flexible, and high-quality approach for multi-image MRI super-resolution without retraining.

\section{Background}
\label{sec:background}

\subsection{Diffusion models}
\label{sec:diffusion}
Diffusion-based models are the current state of the art in generative modeling~\cite{cao2024survey}.
They include the denoising diffusion probabilistic model~(DDPM)~\cite{ho2020denoising}, the denoising diffusion implicit model~(DDIM)~\cite{song2021denoising}, score-based models~\cite{dps,chung2022improving}, Schr\"odinger bridges~\cite{chen2022likelihood,shi2023diffusion,su2023dual}, and Brownian bridges~\cite{li2023bbdm,lee2024ebdm,choo2024slice}.
There are also connections between diffusion-based models and flow matching~\cite{liu2023flow,lipman2023flow,gao2025diffusionmeetsflow}.

Diffusion models establish a Markov chain of transition states between the data distribution and a noise distribution, typically a high-dimensional Gaussian.
By convention, we will denote a state at the data distribution by $x_0$ and a state at the noise distribution by $x_1$.
Thus, arbitrary transition states are denoted by $x_t$ for $t\in[0, 1]$.
A state $x_s$ is closer to the data distribution than $x_t$ whenever $s < t$.
The forward diffusion process describes how data diffuses into noise and is of closed form.
The reverse diffusion process describes how to denoise each transition state and is parameterized by a neural network.
The goal in training a diffusion model is to find parameters $\theta$ that allow the neural network $f_\theta(x_t, t)$ to move from transition state $x_t$ at time $t$ to a less-noisy state $x_s$.
Once trained, at inference time, a sample $x_1$ is drawn and successively transformed through a series of transition states until arriving at the data state $x_0$.

The forward diffusion process defines $x_t$ as a linear combination between data and noise:
\begin{equation}
  x_t = \alpha_t x_0 + \beta_t \epsilon,
\end{equation}
where $\alpha_t$ and $\beta_t$ are parameters for the noise schedule and $\epsilon \sim \mathcal{N}(\mathbf{0}, \mathbf{I})$ by definition.
The reverse diffusion process is the following:
\begin{equation}
  x_s = \alpha_s \hat{x}_0 + \beta_s \hat{\epsilon},
\label{eq:ddpm-ddim}
\end{equation}
where $\hat{\epsilon} = (x_t - \alpha_t \hat{x}_0) / \beta_t$ and $\hat{x}_0$ is estimated by the diffusion network $f_\theta$.
Equation~\ref{eq:ddpm-ddim} is the DDIM sampler~\cite{DDIM}, which generalizes the DDPM framework~\cite{ho2020denoising} and allows accelerated and deterministic sampling.
The full reverse diffusion process arrives at $t=0$ by iteratively applying Eq.~\ref{eq:ddpm-ddim}.

While the methods presented in this paper are agnostic to the specific way $\hat{x}_0$ is obtained~(sample estimation, noise estimation, velocity estimation, etc.~\cite{gao2025diffusionmeetsflow}), our work will specifically make use of the flow estimation framework~\cite{liu2023flow,lipman2023flow,gao2025diffusionmeetsflow}.
That is, the network $f_\theta$ estimates the flow-matching field: $\hat{u} = \hat{\epsilon} - \hat{x}_0 = f_\theta(x_t, t)$, so
\begin{equation}
  \hat{x}_0 = \hat{\epsilon} - f_\theta(x_t, t)
\label{eq:flow_x0}
\end{equation}
and
\begin{equation}
  \hat{\epsilon} = f_\theta(x_t, t) + \hat{x}_0.
\label{eq:flow_eps}
\end{equation}

Due to the empirical success of diffusion models, much theoretical work has been done to connect schools of thought.
There are now many leading perspectives on diffusion models, including the score, variational, flow-matching, and Schr\"odinger bridge perspectives.

\subsection{2D Magnetic Resonance Image Acquisition}
\label{sec:mri_intro}
In MRI, image volumes are rarely obtained alone.
Usually, imaging sessions acquire multiple images of different contrasts, orientations, and resolutions.
We will consider the common scenario of an MRI session where the same pulse sequence is used to acquire multiple images along orthogonal axes, each potentially with different slice thicknesses.
In this case, image volumes are formed by \textit{2D acquisition}~(also called \textit{2D MRI} or \textit{2D multi-slice MRI}).
In 2D MRI acquisitions, samples are acquired in 2D k-space after exciting a slab within the field of view using slice selection~\cite{brown2014magnetic}.
Different regions are selected by varying the slice selection parameters, resulting in a set of independently acquired slices.
Each of these slices is converted to the image domain with the inverse Fourier transform and then stacked along the through-plane axis to create a 3D volume.

Slice selection determines the through-plane resolution.
The physics of slice selection allow us to write the relationship between HR and LR images along one dimension.
In Eq.~\ref{eq:forward}, $A$ can be written as a strided convolution using the slice selection profile and the separation between slices in millimeters.
In practice, the slice selection profile is proprietary and its closed-form equation is unknown.
However, it is known that a Gaussian kernel serves as an reliable approximation.
In this work, we assume that the slice selection uses Gaussian profiles with slice separations and full-with-at-half-maxes equal to the scale factor. 

\subsection{Diffusion posterior sampling for SISR}
\label{sec:sisr_mri}
Diffusion Posterior Sampling~(DPS)\cite{dps} uses Jensen's approximation for the likelihood term in Eq.~\ref{eq:bayes}: $p(y\mid x_t)\approx p(y\mid\mathbb{E}[x_0|x_t])$.
This is convenient, since computing $\mathbb{E}[x_0|x_t]$ is a conventional step in the reverse diffusion process~(see Eq.~\ref{eq:flow_x0}).
With this, DPS adjusts samples $x_s$ in the diffusion process towards the data-consistent manifold:
\begin{equation}
    x_s -\nabla_{x_t}\|A\mathbb{E}[x_0|x_t]-y\|.
\end{equation}
The algorithm for implementing DPS is shown in Algorithm~\ref{alg:dps}.
It should be noted that DPS does not require the forward operator to be linear, only differentiable.

Many methods have built on this initial formulation from DPS.
We defer the reader to \cite{chung2025diffusionreview} for further reading, but briefly summarize a few methods here.
The denoising diffusion null-space model~(DDNM) does a state-wise correction with a hard projection for linear forward operators.
Diffusion plug-and-play image restoration~(DiffPIR) uses the proximal algorithm half-quadratic splitting to update diffusion states to be proximal to the observation.
Diffusion posterior proximal sampling~(DPPS) samples multiple candidates at each state $x_t$ and selects the one with the smallest distance to the data-consistent manifold.
Diffusion Maximum A Posteriori~(DMAP) performs multiple sampling steps to better approximate a MAP sample.

\section{Method}
\label{sec:method}
Our objective is to generalize diffusion priors to the multi-image super-resolution~(MISR) problem.
We first formalize the MISR observation model and derive a separable likelihood expression that permits multi-image data consistency without constructing a joint operator.
We then analyze the implications of this decomposition for DPS and describe MISR extensions of several existing solvers, including DPS, DMAP, DPPS, and diffusion-based PnP methods.

\subsection{MISR Preliminaries}
\label{sec:prelim}
We seek a single HR image $x\in\mathbb{R}^n$ that is consistent with multiple LR measurements $\{y_i\in\mathbb{R}^{m_i}\}_{i=1}^N$.
This yields a set of equations:
\begin{equation}
    y_i=A_ix+\epsilon_i,
\label{eq:multi-forward}
\end{equation}
where $A_i\in\mathbb{R}^{m_i\times n}$ and $\epsilon_i\sim\mathcal{N}(0, \sigma_i \mathbf{I}_{m_i})$.
Since all operations are linear, it is possible to rewrite Eq.~\ref{eq:multi-forward} as
\begin{equation}
    \underline{y} = \underline{A}x + \underline{\epsilon},
\end{equation}
where $\underline{y} = \begin{bmatrix} y_1 & \ldots & y_N \end{bmatrix}^\top$, $\underline{A} = \begin{bmatrix} A_1 & \ldots & A_N \end{bmatrix}^\top$, and $\underline{\epsilon} = \begin{bmatrix} \epsilon_1 & \ldots & \epsilon_N \end{bmatrix}^\top$.
Although \underline{A} is mathematically well-defined, in MRI each $A_i$ may operate on different grids with different slice geometries, making explicit stacking computationally impractical.

\subsection{Likelihood separability in DPS}
\label{sec:separability}
The joint negative log-likelihood of Eq.~\ref{eq:multi-forward} is
\begin{equation}
    \mathcal{L}(x) = \sum_{i=1}^N \frac{1}{2\sigma_i^2}\|A_ix-y_i\|_2^2,
\end{equation}
which is separable across observations.

\begin{proposition}\label{prop:separability}
Let $\mu_0(x_t) = \mathbb{E}[x_0\mid x_t]$ be the diffusion model's estimation of the data sample from time $t$. 
Then, the likelihood gradient to correct $x_t$ towards the data consistent manifold is
\begin{equation}
    \nabla_{x_t} \mathcal{L}(\mu_0(x_t)) = \sum_{i=1}^N \frac{1}{\sigma_i^2}\nabla_{x_t}\|A_i \mu_0(x_t) -y_i\|_2^2.
\label{eq:likelihood_sep}
\end{equation}
In particular, each $A_i$ contributes an independent correction direction, and the combined gradient is their sum.
\end{proposition}

\begin{proof}
By independence of measurements,
\begin{equation}
    \mathcal{L}(\mu_0) = \sum_{i=1}^N\mathcal{L}_i(\mu_0).
\end{equation}
Since the diffusion network $f_\theta$ is used only to generate $\mu_0(x_t)$ and the measurement operators $A_i$ are linear and do not couple measurements,
\begin{equation}
    \nabla_{x_t} \mathcal{L}(\mu_0(x_t)) = \sum_{i=1}^N \nabla_{x_t}  \mathcal{L}_i(\mu_0(x_t)).
\end{equation}
With substitution,
\begin{equation}
    \nabla_{x_t} \mathcal{L}(\mu_0(x_t)) = \sum_{i=1}^N \frac{1}{\sigma_i^2}\nabla_{x_t}\|A_i \mu_0(x_t) -y_i\|_2^2
\end{equation}
\end{proof}

\textbf{Remark.}\quad
Proposition~\ref{prop:separability} guarantees that MISR diffusion posterior sampling requires only independent per-measurement gradients without requiring a joint operator, modification of the diffusion model, nor an increase in NFEs.
The nonlinearity of $\mu_0$ does not break linear separability because dependence on $\mu_0$ flows through the chain rule.

The separability of likelihoods under DPS has several implications.
First, as we will demonstrate in this work, we can extend several approaches in the DPS family for MISR.
Second, likelihood separability can generalize to any components that are independent in $A$, not just multiple observations.
For example, in the SISR framework, we could split $A$ into independent rows and consider different regions of images as separate measurements.
This extends to further norm-preserving linear maps, and so block-stacked wavelet transform matrices could allow splitting by wavelet subband.
We leave this exploration for future work.
Finally, likelihood separability allows for independent weights for each component.
With even weights, all measurements are considered equally.
However, some measurements may be less reliable, for example, with more noise than others.

    

\begin{algorithm}[th]
\caption{DPS}
\label{alg:dps}
\begin{algorithmic}[1]
\Require $T$, $A$, $y$, $\zeta_t$
\State $x_1 = \mathcal{N}(0, I)$
\For{$t = 1$ to $0$ with $T$ steps}
    \State $x_s \sim p_\theta(X_s \mid x_t)$
    \State $\mu_0 := \mathbb{E}[X_0 \mid x_t]$
    \State $x_s = x_s - \zeta_t \nabla_{x_t} \| A\mu_0 - y \|$
\EndFor
\State \Return $x_0$
\end{algorithmic}
\end{algorithm}

\begin{algorithm}[th]
\caption{MISR DPS}
\label{alg:misr_dps}
\begin{algorithmic}[1]
\Require $T$, $\{A_i\}_{i=1}^N$, $\{y_i\}_{i=1}^N$, $\{\sigma_i\}_{i=1}^N$, $\zeta_t$
\State $x_1 = \mathcal{N}(0, I)$
\State $w_i := (\sigma_i^2 + \sigma_\mathrm{floor})^{-1}$
\State $w_i := N\cdot w_i / (\sum_{j=1}^Nw_j)$
\For{$t = 1$ to $0$ with $T$ steps}
    \State $x_s \sim p_\theta(X_s \mid x_t)$
    \State $\mu_0 := \mathbb{E}[X_0 \mid x_t]$
    
    \State $x_s = x_s - \zeta_t \nabla_{x_t} \sum_{i=1}^N\left[w_i\cdot\| A_i\mu_0 - y_i \|_2^2\right]$
\EndFor
\State \Return $x_0$
\end{algorithmic}
\end{algorithm}

\textbf{Measurement noise weighting}\quad
Each measurement may have its own noise strength; i.e., $\sigma_i\neq \sigma_j$ for $i\neq j$.
In such scenarios, it is disadvantageous to correct states towards all observations equally.
Thus, we propose to inversely scale the gradients by the strength of the noise for that observation.

\subsection{MISR inverse diffusion methods}
The previous sections provided a framework for MISR that is applicable to existing approaches.
Here, we outline how these methods can be modified for noise-weighted MISR.

\textbf{MISR DPS}\quad
We describe MISR DPS in Algorithm~\ref{alg:misr_dps}.
The gradient update term is weighted by $w_i$: $\nabla_{x_t}\sum_{i=1}^N\left[w_i\cdot\|A_i x_{0\mid t}-y_i\|_2^2\right ]$, where $w_i\propto(\sigma_i^2+\sigma_\mathrm{floor})^{-1}$ and $\sigma_\mathrm{floor}$ is a noise floor that the diffusion model $f_\theta$ is robust to.

\textbf{MISR DMAP}\quad
DMAP uses multiple samples and gradient steps to refine the approximation of state $s$ from state $t$.
Since it is algorithmically similar to DPS~(compare DPS in Algorithm~\ref{alg:dps} to DMAP in Algorithm~\ref{alg:dmap}), we are also able to generalize DMAP for MISR, shown in Algorithm~\ref{alg:misr_dmap}.

\textbf{MISR DPPS}\quad
To implement MISR for DDPS, the selected candidate must have the smallest distance to all data-consistent manifolds.

\textbf{MISR PnP}\quad
Plug-and-play~(PnP) methods that rely on priors are also amenable to MISR.
The alternating direction method of multipliers~(ADMM) uses variable splitting to optimize for data-consistency and prior-consistency~\cite{boyd2011distributed}.
The data-consistency term is still written as the squared $\ell_2$ norm, so Eq.~\ref{eq:likelihood_sep} is still applicable.
The same logic applies to DiffPIR, with half-quadratic splitting instead of ADMM.

\begin{algorithm}[th]
\caption{DMAP}
\label{alg:dmap}
\begin{algorithmic}[1]
\Require $T$, $K$, $A$, $y$, $\zeta_t$, $d$
\State $x_1 = \mathcal{N}(0, I)$
\For{$t = 1$ to $0$ with $T$ steps}    
    \State $x_s \sim p_\theta(X_s \mid x_t)$
    \State $\mu_s = \mathbb{E}[X_s \mid x_t]$
        \For{$j = 1$ to $K$} 
            \State $\mu_0 := \mathbb{E}[X_0 \mid x_s]$
            \State $x_s = x_s - \zeta_t \nabla_{x_s} \| A\mu_0 - y \|$
            \State $x_s = \mu_s - \sqrt{d}\sigma_t\frac{x_s - \mu_s}{\|x_s - \mu_s\|_2^2}$
        \EndFor
\EndFor
\State \Return $x_0$
\end{algorithmic}
\end{algorithm}

\begin{algorithm}[th]
\caption{MISR DMAP}
\label{alg:misr_dmap}
\begin{algorithmic}[1]
\Require $T$, $K$, $\{A_i\}_{i=1}^N$, $\{y_i\}_{i=1}^N$, $\{\sigma_i\}_{i=1}^N$, $\zeta_t$, $d$
\State $x_1 = \mathcal{N}(0, I)$
\State $w_i := (\sigma_i^2 + \sigma_\mathrm{floor})^{-1}$
\State $w_i := N\cdot w_i / (\sum_{j=1}^Nw_j)$
\For{$t = 1$ to $0$ with $T$ steps}
    \State $x_s \sim p_\theta(X_s \mid x_t)$
    \State $\mu_s = \mathbb{E}[X_s \mid x_t]$
        \For{$j = 1$ to $K$} 
            \State $\mu_0 := \mathbb{E}[X_0 \mid x_s]$
            \State $x_s = x_s - \zeta_t \nabla_{x_s}  \sum_{i=1}^N\left[w_i\cdot\| A_i\mu_0 - y_i \|_2^2\right]$
            \State $x_s = \mu_s - \sqrt{d}\sigma_t\frac{x_s - \mu_s}{\|x_s - \mu_s\|_2^2}$
        \EndFor
\EndFor
\State \Return $x_0$
\end{algorithmic}
\end{algorithm}

\section{Experiments and Results}
\label{sec:experiments}
Our experiments focus on super-resolution for anisotropic magnetic resonance head images.
We consider images that are $4\times$, $8\times$, and $16\times$ worse resolution along through-plane axes than in-plane axes.
For HR image volumes of $1\mathrm{~mm}^3$ isotropic resolution, these \textit{scale factors} correspond to $4$, $8$, and $16$ mm slice thicknesses.
To enable quantitative reference-based distortion metrics, we perform simulations of data using Gaussian slice selection profiles with full-width-at-half-max and slice separation equal to the scale factor.
Following Eq.~\ref{eq:multi-forward}, we also add Gaussian noise with $\sigma_i$ proportional to the voxel size due to the physics of how noise interacts with slice selection~\cite{prince2006medical}.
Thicker slices have less noise.
Specifically, we chose $\sigma_i=0.1 / k_i$, where $k_i$ is the scale factor for the LR image volume $y_i$.

\begin{figure*}[ht]
    \centering
    \includegraphics[width=0.95\linewidth]{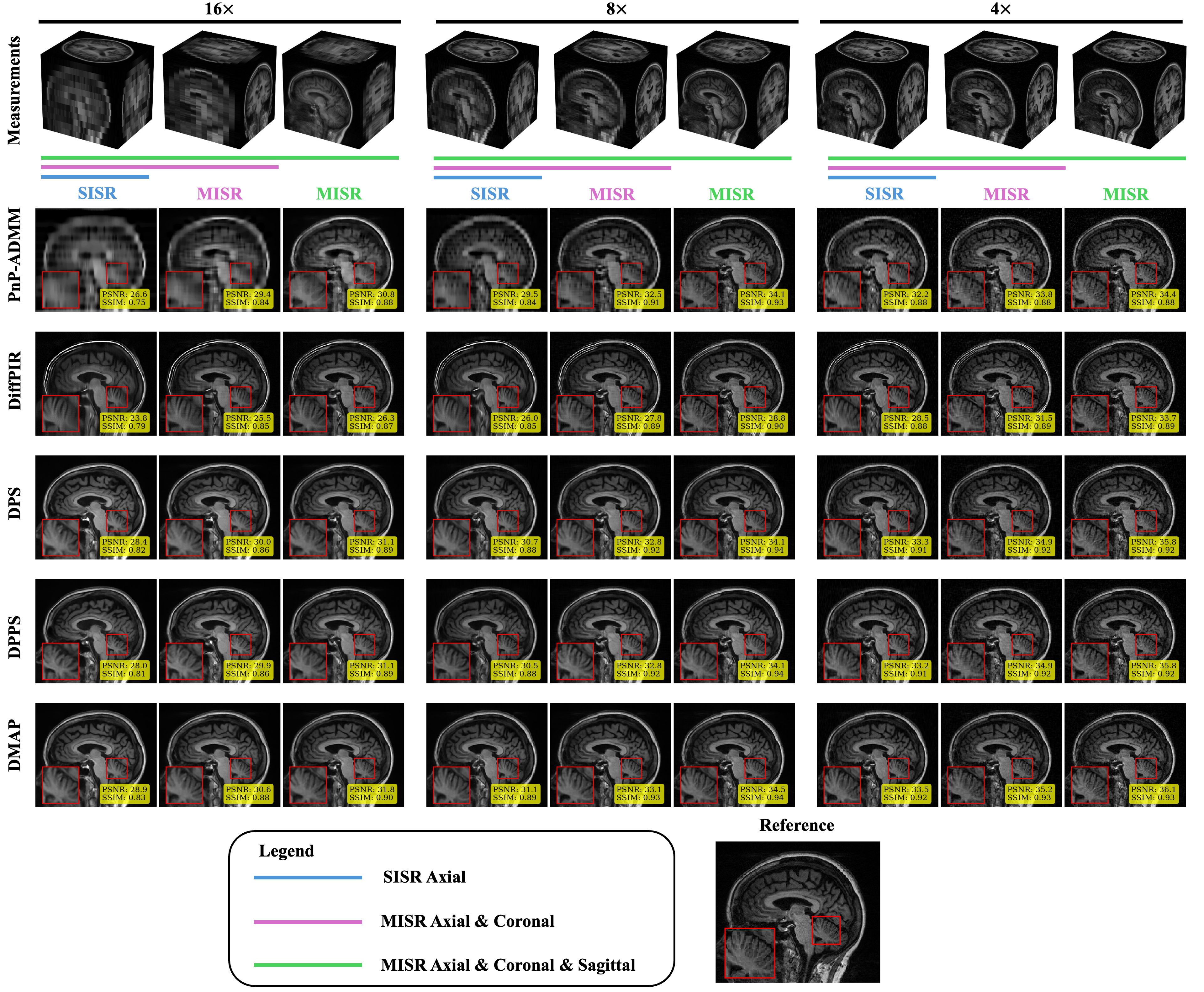}
    \caption{Qualitative results for a representative subject from the AIBL dataset. Row-wise labels designate LR inputs as ``Measurements'' and SR estimations named by method. Supercolumns group scale factors together. Within supercolumns, each column corresponds to SISR using only the axial acquisition as input~(blue), then MISR using axial and coronal~(pink), and MISR using axial, coronal, and sagittal~(green). The HR reference is shown at the bottom of the figure. For all methods, the same area highlighting the folia of the cerebellum is zoomed into the inset.}
    \label{fig:qual}
\end{figure*}

We used an open-source pre-trained non-latent 3D diffusion model for brain MRI~\cite{remedios2025diffusion}.
This model is a flow-estimation model, and we used $64$ DDIM steps for all methods and all experiments.
This model was trained on more than $70,000$ $T_1$-weighted brain volumes from 36 open datasets, and additionally withheld two sites from training: AIBL~\cite{dataset-aibl} and SLEEP~\cite{dataset-sleep}.
Thus, to avoid model bias, we selected 50 subjects with $1\mathrm{~mm}^3$ resolution from each of those two sites are our data cohort.
For preprocessing, these image volumes were padded and/or cropped to a voxel size of $192\times224\times192$ using the center-of-mass location from an HD-BET~\cite{Isensee-hdbet}-calculated brain mask, then linearly normalized to the intensity range $[-1, 1]$ using the volume's minimum and maximum values.

We conducted two experiments.
First, we compared SISR to MISR with orthogonal planes.
Then, we perform an ablation study and compare noise-weighting to uniform weighting in MISR DPS and MISR DMAP.
For evaluation, we used PSNR, SSIM, and when comparing SISR to MISR we additionally calculated the Frechet Inception Distance~(FID)~\cite{heusel2017gans,Seitzer2020FID}.
To enable the computation of FID for 3D volumetric images, we extract 2D slices from each volume from all three cardinal planes, separated by 4 mm, and spanning the center 128 mm of the volume.
These slices were linearly normalized to the intensity range $[0, 255]$ using the volume's minimum and maximum values, quantized to 8-bit integers, and saved to disk as \texttt{.png} files.
The reference dataset for the FID computations for the AIBL cohort was the SLEEP cohort, and vice-versa.
This was done to prevent an abnormally low FID score due to similar intensity statistics between, for example, true-HR-AIBL and super-resolved-AIBL images.

\begin{figure*}[ht]
    \centering
    \includegraphics[width=\linewidth]{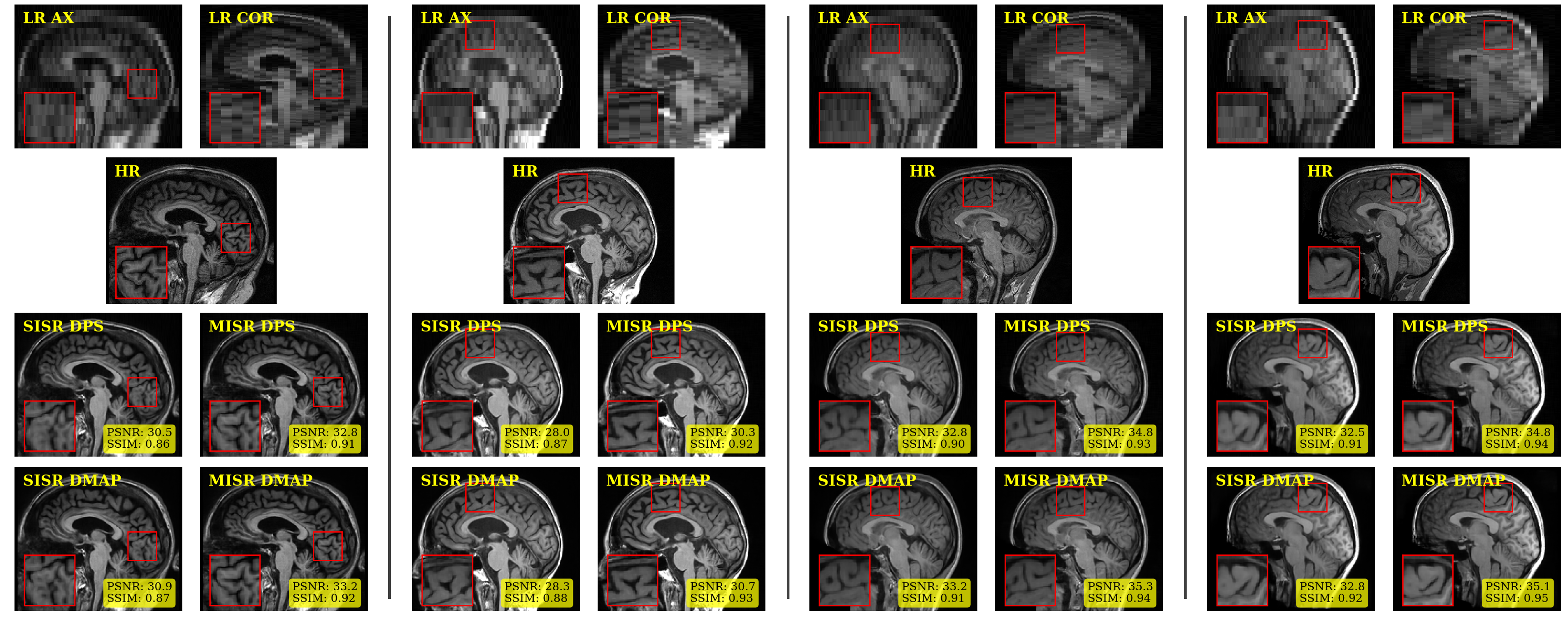}
    \caption{Sagittal slices from four representative subjects are shown for $8\times$ scale factor LR inputs. Each displayed image is labeled for its contents. LR: low-resolution; AX: axial acquisition; COR: coronal acquisition; HR: ground-truth high-resolution; SISR: single-image super-resolution, using the axial acquisition as input; MISR: multi-image super-resolution, using the axial and coronal acquisitions as input. Zoomed inset regions are located by the red box, highlighting anatomical differences that are not correctly recovered by SISR.}
    \label{fig:multi-qual}
\end{figure*}

\subsection{SISR vs MISR}
In Fig.~\ref{fig:teaser} and Fig.~\ref{fig:qual}, we show qualitative comparisons between slices of representative volumes for SISR and MISR across several methods.
We emphasize that all methods operate directly in 3D, and 2D slices are shown for clarity.
For MISR scenarios, there are more LR images to use as input, leading to improved results for all methods and scale factors.

In Fig.~\ref{fig:multi-qual}, we show four more representative subjects at $8\times$ scale factor.
With a second observation, these regions have more correct anatomy.
However, the distortion metrics do not reflect geometric or anatomical errors; the approximately 2 to 3 dB difference in PSNR is not sufficient to describe the improvements gained from another measurement.


\begin{sidewaystable*}[htbp]
\caption{Quantitative results for 100 subjects. For PSNR and SSIM, the mean$\pm$std. are reported. FID is computed as mentioned in Sec.~\ref{sec:experiments} Best results by scale factor are \textbf{bolded} and second-best results are \underline{underlined}.}
\label{tab:quantitative_results}
\centering
\renewcommand{\arraystretch}{1.2}
\setlength{\tabcolsep}{6pt}
\begin{tabular}{c|l|ccc|ccc|ccc}
\toprule
& & \multicolumn{3}{c|}{\textbf{SISR}} & \multicolumn{6}{c}{\textbf{MISR}} \\
\cmidrule(lr){3-5} \cmidrule(lr){6-11}
& & \multicolumn{3}{c|}{Axial} & \multicolumn{3}{c|}{Axial \& Coronal} & \multicolumn{3}{c}{Axial \& Coronal \& Sagittal} \\
\cmidrule(lr){3-5} \cmidrule(lr){6-8} \cmidrule(lr){9-11}
\textbf{Scale} & \textbf{Method} & PSNR ($\uparrow$) & SSIM ($\uparrow$) & FID ($\downarrow$) & PSNR ($\uparrow$) & SSIM ($\uparrow$) & FID ($\downarrow$) & PSNR ($\uparrow$) & SSIM ($\uparrow$) & FID ($\downarrow$) \\
\midrule
\multirow{5}{*}{\textbf{$\,4\times\,$}} & PnP-ADMM & $31.39 \pm 1.42$ & $0.880 \pm 0.014$ & $98.65$ & $33.30 \pm 1.11$ & $0.883 \pm 0.009$ & $84.97$ & $34.17 \pm 0.91$ & $0.884 \pm 0.008$ & $85.38$ \\
 & DiffPIR & $27.92 \pm 1.26$ & $0.888 \pm 0.014$ & $65.44$ & $30.23 \pm 1.19$ & $0.895 \pm 0.009$ & $58.15$ & $32.14 \pm 1.25$ & $0.896 \pm 0.007$ & $53.97$ \\
 & DPS & $\underline{32.78 \pm 2.03}$ & $\underline{0.918 \pm 0.017}$ & $41.28$ & $\underline{34.59 \pm 1.73}$ & $\underline{0.928 \pm 0.012}$ & $47.50$ & $35.52 \pm 1.44$ & $\underline{0.927 \pm 0.009}$ & $\underline{51.76}$ \\
 & DPPS & $32.76 \pm 2.04$ & $\underline{0.918 \pm 0.017}$ & $\underline{41.19}$ & $34.59 \pm 1.74$ & $\underline{0.928 \pm 0.012}$ & $\underline{47.30}$ & $\underline{35.54 \pm 1.42}$ & $\underline{0.927 \pm 0.009}$ & $51.83$ \\
 & DMAP & $\mathbf{33.08 \pm 2.14}$ & $\mathbf{0.929 \pm 0.018}$ & $\mathbf{32.07}$ & $\mathbf{34.97 \pm 1.91}$ & $\mathbf{0.940 \pm 0.014}$ & $\mathbf{34.77}$ & $\mathbf{35.98 \pm 1.64}$ & $\mathbf{0.940 \pm 0.012}$ & $\mathbf{37.18}$ \\
\midrule
\multirow{5}{*}{\textbf{$\,8\times\,$}} & PnP-ADMM & $28.70 \pm 1.50$ & $0.844 \pm 0.027$ & $158.53$ & $31.51 \pm 1.59$ & $0.914 \pm 0.017$ & $85.09$ & $33.27 \pm 1.67$ & $0.936 \pm 0.013$ & $83.50$ \\
 & DiffPIR & $25.18 \pm 1.29$ & $0.857 \pm 0.024$ & $91.66$ & $26.91 \pm 1.25$ & $0.898 \pm 0.016$ & $78.43$ & $27.82 \pm 1.30$ & $0.910 \pm 0.013$ & $70.85$ \\
 & DPS & $\underline{30.04 \pm 2.11}$ & $\underline{0.878 \pm 0.027}$ & $39.58$ & $\underline{32.21 \pm 2.07}$ & $\underline{0.922 \pm 0.019}$ & $36.25$ & $\underline{33.65 \pm 2.13}$ & $\underline{0.940 \pm 0.016}$ & $34.38$ \\
 & DPPS & $30.00 \pm 2.12$ & $0.877 \pm 0.027$ & $\underline{39.03}$ & $32.18 \pm 2.08$ & $0.922 \pm 0.019$ & $\underline{35.84}$ & $33.64 \pm 2.15$ & $\underline{0.940 \pm 0.016}$ & $\underline{34.26}$ \\
 & DMAP & $\mathbf{30.36 \pm 2.16}$ & $\mathbf{0.888 \pm 0.026}$ & $\mathbf{38.97}$ & $\mathbf{32.63 \pm 2.13}$ & $\mathbf{0.931 \pm 0.019}$ & $\mathbf{35.51}$ & $\mathbf{34.17 \pm 2.22}$ & $\mathbf{0.947 \pm 0.016}$ & $\mathbf{32.66}$ \\
\midrule
\multirow{5}{*}{\textbf{$\,16\times\,$}} & PnP-ADMM & $25.82 \pm 1.33$ & $0.751 \pm 0.037$ & $217.75$ & $28.53 \pm 1.62$ & $0.843 \pm 0.028$ & $126.23$ & $29.98 \pm 1.75$ & $0.881 \pm 0.024$ & $110.77$ \\
 & DiffPIR & $23.07 \pm 1.21$ & $0.793 \pm 0.031$ & $116.23$ & $24.68 \pm 1.19$ & $0.856 \pm 0.022$ & $96.28$ & $25.46 \pm 1.20$ & $0.882 \pm 0.019$ & $85.93$ \\
 & DPS & $\underline{27.56 \pm 1.95}$ & $\underline{0.811 \pm 0.036}$ & $44.61$ & $\underline{29.26 \pm 1.99}$ & $\underline{0.859 \pm 0.028}$ & $43.93$ & $\underline{30.50 \pm 2.08}$ & $\underline{0.888 \pm 0.024}$ & $42.75$ \\
 & DPPS & $27.47 \pm 1.94$ & $0.810 \pm 0.037$ & $\mathbf{43.98}$ & $29.25 \pm 2.01$ & $\underline{0.859 \pm 0.029}$ & $\mathbf{42.99}$ & $30.47 \pm 2.08$ & $\underline{0.888 \pm 0.024}$ & $\underline{41.96}$ \\
 & DMAP & $\mathbf{28.04 \pm 2.03}$ & $\mathbf{0.826 \pm 0.036}$ & $\underline{44.43}$ & $\mathbf{29.87 \pm 2.03}$ & $\mathbf{0.877 \pm 0.026}$ & $\underline{43.50}$ & $\mathbf{31.25 \pm 2.14}$ & $\mathbf{0.905 \pm 0.023}$ & $\mathbf{41.89}$ \\
\bottomrule
\end{tabular}
\end{sidewaystable*}

We show quantitative results for our 100 subjects in Table~\ref{tab:quantitative_results}.
In correspondence with the qualitative results from Fig.~\ref{fig:qual}, DMAP outperforms all other methods across the board, both in distortion metrics and FID.
An interesting result in Table~\ref{tab:quantitative_results} is the metrics' trends when adding more LR images.
As expected, PSNR and SSIM improve in all cases.
However, for the $4\times$ scale factor, the FID score generally worsens for the DPS family of methods as more images are added.
This is potentially due to the perception-distortion tradeoff~\cite{blau2018perception}; since the SR image $\hat{x}$ is guided towards being consistent with multiple images, its freedom in being a realistic image is restricted.
However, this phenomenon is not present at $8\times$ and $16\times$ scale factors.

\subsection{Ablation: noise-weighting}
\label{sec:ablation}
In the previous experiment, each of the LR images used for MISR had the same scale factor.
The noise present in the images had the same power, so the noise weights $w_i$ were equal for all $i$ and essentially absorbed into the gradient scalar $\zeta_t$.
To evaluate the efficacy of the proposed noise weighting, we compared with and without noise-weighting for mixed-resolution inputs: $4\times$ and $16\times$ scale factors for two orthogonal inputs.
We used MISR DPS and MISR DMAP for this evaluation and varied the noise by $\sigma_i=\sigma_\mathrm{base} / k_i$, for $\sigma_\mathrm{base}=0.15, 0.3, 0.45$.
These values correspond to approximately $1\times$, $2\times$, and $3\times$ the average intensity standard deviation in HR images~(after normalization to $[-1, 1]$).

In Table~\ref{tab:ablation}, we show the results of this experiment.
For both MISR DPS and MISR DMAP, noise-weighting improves all metrics.
As the noise level becomes more severe, the effect of noise weighting becomes larger.

\begin{table}[htbp]
\centering
\caption{Quantitative results for our ablation without noise weighting~(w/o NW) and with noise weighting~(w/ NW) for the two-image $4\times$ and $16\times$ scale factors for both MISR DPS and MISR DMAP.}
\label{tab:ablation}
\setlength{\tabcolsep}{3pt}

\begin{tabular}{l|c|cc|cc}
\toprule
 & & \multicolumn{2}{c|}{\textbf{w/o NW}} & \multicolumn{2}{c}{\textbf{w/ NW}} \\
 & $\sigma_\mathrm{base}$ & PSNR ($\uparrow$) & SSIM ($\uparrow$) & PSNR ($\uparrow$) & SSIM ($\uparrow$) \\
\midrule

\multirow{3}{*}{\rotatebox[origin=c]{90}{\textbf{DPS}}}
 & 0.15 & $32.5 \pm 2.4$ & $0.90 \pm 0.01$ & $32.5 \pm 2.4$ & $0.90 \pm 0.01$ \\
 & 0.3 & $29.8 \pm 1.6$ & $0.78 \pm 0.01$ & $29.9 \pm 1.6$ & $0.79 \pm 0.01$ \\
 & 0.45 & $27.0 \pm 1.1$ & $0.63 \pm 0.02$ & $27.5 \pm 1.2$ & $0.66 \pm 0.02$ \\
\midrule

\multirow{3}{*}{\rotatebox[origin=c]{90}{\textbf{DMAP}}}
 & 0.15 & $32.9 \pm 2.6$ & $0.91 \pm 0.02$ & $32.9 \pm 2.6$ & $0.91 \pm 0.02$ \\
 & 0.3 & $30.0 \pm 1.6$ & $0.80 \pm 0.01$ & $30.1 \pm 1.7$ & $0.80 \pm 0.01$ \\
 & 0.45 & $27.2 \pm 1.2$ & $0.65 \pm 0.01$ & $27.5 \pm 1.3$ & $0.67 \pm 0.01$ \\
\bottomrule
\end{tabular}

\end{table}

\section{Conclusion}
\label{sec:conclusion}
We presented an approach for MISR using diffusion models.
We showed that likelihood separation underneath the diffusion trajectory allows for a flexible multi-image framework and implemented it with noise-weighting for the diffusion posterior sampling family of inverse problem solvers.
We conducted experiments at various scale factors with orthogonal acquisition planes and showed that MISR outperforms SISR both qualitatively and quantitatively, and additionally better preserves anatomy.
Future work will address misalignment between measurements~(though an estimated registration can be folded in to the measurement $A$) as well as multi-contrast acquisitions.
To the best of our knowledge, the results of this work are state-of-the-art in the field of MRI super-resolution for anisotropic image volumes.
Our results indicate that MISR diffusion could make standard 2D MRI acquisitions usable in all viewing planes, potentially reducing the need for long, thin-slice 3D scans and restoring image analytic quality for retrospectively accessed imaging sessions.
\clearpage
{
    \small
    \bibliographystyle{ieeenat_fullname}
    \bibliography{main}
}


\end{document}